\documentclass[conference,letterpaper]{IEEEtran}
\usepackage[noadjust]{cite}
\usepackage{amsmath,amssymb,amsfonts}
\usepackage{xcolor}
\def\BibTeX{{\rm B\kern-.05em{\sc i\kern-.025em b}\kern-.08em
		T\kern-.1667em\lower.7ex\hbox{E}\kern-.125emX}}

\usepackage{mleftright}
\mleftright
\usepackage{times}
\usepackage{epsfig}
\usepackage{graphicx}
\usepackage{verbatim}
\usepackage{url}
\usepackage[T1]{fontenc}
\usepackage{amsthm}
\usepackage{array}
\usepackage{floatflt}
\usepackage{bbm}
\usepackage{xcolor,colortbl}
\makeatletter

\providecommand{\tabularnewline}{\\}

\theoremstyle{plain}
\newtheorem{thm}{\protect\theoremname}
\theoremstyle{definition}

\theoremstyle{plain}

\theoremstyle{remark}

\usepackage{algorithm,algpseudocode}
\usepackage[utf8]{inputenc}
\usepackage{textcomp}
\usepackage{everyshi}
\makeatletter
\@namedef{ver@everyshi.sty}{}
\makeatother

\usepackage{tikz}
\usepackage{booktabs}

\usepackage[symbol]{footmisc}
\usepackage{adjustbox}
\usepackage{placeins}
\usepackage{rotating}
\usepackage{pifont}
\usepackage{wrapfig}

\usetikzlibrary{automata,positioning,arrows}
\usepackage[mode=buildnew]{standalone}
\usepackage{pgfplots}
\usepackage{tikz-3dplot}
\usetikzlibrary{positioning}
\usepgfplotslibrary{groupplots,units}
\pgfplotsset{compat=1.14}

\usepackage{enumitem}
\usepackage{diagbox}

\DeclareMathOperator{\argmin}{idxmin}

\DeclareMathOperator{\simi}{sim}

\DeclareMathOperator{\support}{supp}

\makeatother

\providecommand{\definitionname}{Definition}
\providecommand{\lemmaname}{Lemma}
\providecommand{\remarkname}{Remark}
\providecommand{\theoremname}{Theorem}
\newcolumntype{P}[1]{>{\centering\arraybackslash}p{#1}}

\makeatletter
\newcommand\ztag[1]{%
	\def\@currentlabel{#1}%
	\gdef\tmp{%
		\addtocounter{equation}{-1}%
		\def\theequation{#1}}%
	\aftergroup\aftergroup\aftergroup\aftergroup\aftergroup\aftergroup
	\aftergroup\aftergroup\aftergroup\aftergroup\aftergroup\aftergroup
	\aftergroup\aftergroup\aftergroup\aftergroup\aftergroup\aftergroup
	\aftergroup\aftergroup\aftergroup\aftergroup\aftergroup\aftergroup
	\aftergroup\aftergroup\aftergroup\aftergroup\aftergroup\aftergroup
	\aftergroup
	\tmp}
\makeatother
\usepackage[breaklinks=true,colorlinks,bookmarks=false]{hyperref}
\newcommand\copyrighttext{%
	\footnotesize \copyright~2021 IEEE. Personal use of this material is permitted. Permission from IEEE must be obtained for all other uses, in any current or future media, including reprinting/republishing this material for advertising or promotional purposes,creating new collective works, for resale or redistribution to servers or lists, or reuse of any copyrighted component of this work in other works.}
\newcommand\copyrightnotice{%
	\begin{tikzpicture}[remember picture,overlay]
	\node[anchor=south,yshift=10pt] at (current page.south) {\fbox{\parbox{\dimexpr\textwidth-\fboxsep-\fboxrule\relax}{\copyrighttext}}};
	\end{tikzpicture}%
}
\begin{document}
	\title{COPS: Controlled Pruning Before Training Starts
	}
	
	\author{\IEEEauthorblockN{Wimmer Paul}
		\IEEEauthorblockA{\textit{Image Processing} \\
			\textit{Robert Bosch GmbH \& L\"ubeck University}\\
			71229 Leonberg, Germany \\
			\{Paul.Wimmer,}
		\and
		\IEEEauthorblockN{Mehnert Jens}
		\IEEEauthorblockA{\textit{Image Processing} \\
			\textit{Robert Bosch GmbH}\\
			71229 Leonberg, Germany \\
			JensEricMarkus.Mehnert,}
		\and
		\IEEEauthorblockN{Condurache Alexandru}
		\IEEEauthorblockA{\textit{Engineering Cognitive Systems} \\
			\textit{Robert Bosch GmbH \& L\"ubeck University}\\
			70499 Stuttgart, Germany \\
			AlexandruPaul.Condurache\}@de.bosch.com}
	}
	
	\maketitle
	
	%%%%%%%%% ABSTRACT
	\begin{abstract}
		State-of-the-art deep neural network
		(DNN) pruning techniques, applied one-shot before training starts, evaluate sparse architectures
		with the help of a single criterion---called pruning score. 
		%as for instance the saliency score \cite{lee_2018} or the
		%importance score \cite{wang_2020}. %The network is pruned
		%such that the smallest possible score is achieved. 
		Pruning weights based on a solitary score works well for some architectures and pruning
		rates but may also fail for other ones. As a common baseline for pruning scores,
		we introduce the notion of a generalized synaptic score (GSS).
		In this work we do not concentrate on a single pruning criterion,
		but provide a framework for combining arbitrary GSSs to create more powerful pruning strategies. These \emph{CO}mbined \emph{P}runing \emph{S}cores
		(COPS) are obtained by solving a constrained optimization problem. Optimizing for
		more than one score prevents the sparse network to overly specialize on
		an individual task, thus COntrols Pruning before training Starts. The combinatorial optimization problem given by COPS
		is relaxed on a linear program (LP). This LP is solved analytically and determines a solution for COPS. %convex set, solved analytically and used to determine
		%a solution of the un-relaxed problem. 
		Furthermore, an algorithm to compute it for
		two scores numerically is proposed and evaluated. Solving COPS in such a way has lower complexity than the best general LP solver.
		In our experiments we compared pruning with COPS against state-of-the-art methods for different network architectures and image classification tasks and obtained improved results.
		
	\end{abstract}
	\copyrightnotice
	\section{Introduction}
	\makeatletter
	
	\global\long\def\ie{\emph{i.e.} }
	\global\long\def\Ie{\emph{I.e.} }
	\global\long\def\eg{\emph{e.g.} }
	\global\long\def\wrt{w.r.t. }
	\makeatother
	\global\long\def\dl{\nabla L(W)}%
	\global\long\def\mw{m\odot W}%
	\global\long\def\skp#1#2#3{\left\langle#1,#2\right\rangle_{#3}}%
	\global\long\def\id{\text{id}}%
	\global\long\def\DL{\Delta L(W)}%
	\global\long\def\r{\mathbb{R}}%
	\global\long\def\tk{\tilde{\kappa}}%
	\global\long\def\b#1#2#3{\mathcal{B}_{#1,#2}^{#3}}%
	\global\long\def\ball#1#2{\mathcal{B}_{#1}^{#2}}%
	\global\long\def\SD{\{1,\ldots,D\}}%
	\global\long\def\interior#1{\kern0pt  #1^{\mathrm{o}}}%
	\global\long\def\X{\mathcal{X}_{0,\sigma}}%
	\global\long\def\la{\lambda^{\ast}}%
	\global\long\def\ma{m^{\ast}}%
	\global\long\def\ia{I_{\ast}}%
	\global\long\def\pa{P_{\ast}}%
	\global\long\def\ila{I_{\la}}%
	\global\long\def\l{\lambda}%
	\global\long\def\il{I_{\l}}%
	\global\long\def\vgeq{\succcurlyeq}%
	\global\long\def\ltimes{\mathcal{L}_{\times}}%
	\global\long\def\xr{\mathcal{X}_{1,\sigma}}%
	\global\long\def\lmin{\lambda_{-}}%
	\global\long\def\lmax{\lambda_{+}}%
	\global\long\def\one{{\bf 1}}%
	\global\long\def\iplus{I_{\lambda,+}}%
	\global\long\def\ak{\alpha_{\kappa}}%
	\global\long\def\k{\kappa}%
	\global\long\def\mt{m\odot\Theta}%
	\global\long\def\imin{I_{\lambda,-}}%
	\global\long\def\s{\sigma}%
	\global\long\def\is{\mathcal{I}_{\leq\sigma}}%
	\global\long\def\ms{\overline{m}}%
	\global\long\def\1{\mathbbm{1}}%
	\renewcommand*{\thefootnote}{\arabic{footnote}}
	\newcommand{\cmark}{\ding{51}}% 
	\newcommand{\xmark}{\ding{55}}%
	
	\makeatletter
	\renewcommand{\ALG@beginalgorithmic}{\footnotesize}
	\makeatother
	
	\emph{Network pruning} \cite{han_2015,lecun_1990,mozer_1989}
	sets parts of a DNN's weights to zero. This can help to reduce
	the model's complexity and memory requirements, speed up inference \cite{blalock_2020} and even lead to better generalization ability for the network \cite{lecun_1990}.
	In recent years, \emph{training} sparse networks, \ie networks with many weights fixed at zero, became of interest to the
	deep learning community \cite{bellec_2018,frankle_2018,lee_2018,tanaka_2020,wang_2020}, providing the potential benefits of reduced runtime and memory requirements not only for inference but also for training. Sparse networks induce sparse gradient computations, but also the problem of weak gradient signals if too many parameters are pruned \cite{tanaka_2020,wang_2020}. One possibility to maintain a strong gradient signal while training only sparse parts of the network can be achieved by \emph{freezing} big portions of the weights during training at their initial value \cite{wimmer_2020}. But likewise, pruned networks can have a sufficient gradient flow, even for high pruning rates, if the pruning criterion is chosen cautiously \cite{tanaka_2020, wang_2020}.
	In this work we will focus on one-shot pruning methods, applied before
	training starts. By one-shot pruning we mean pruning in a single step,
	not iteratively. Thus, no resources are spent for pre-training \cite{frankle_2018},
	iterative pruning \cite{tanaka_2020} or a dynamical
	change of the network during training \cite{bellec_2018}. 
	\begin{figure}
		\includegraphics[width=\linewidth]{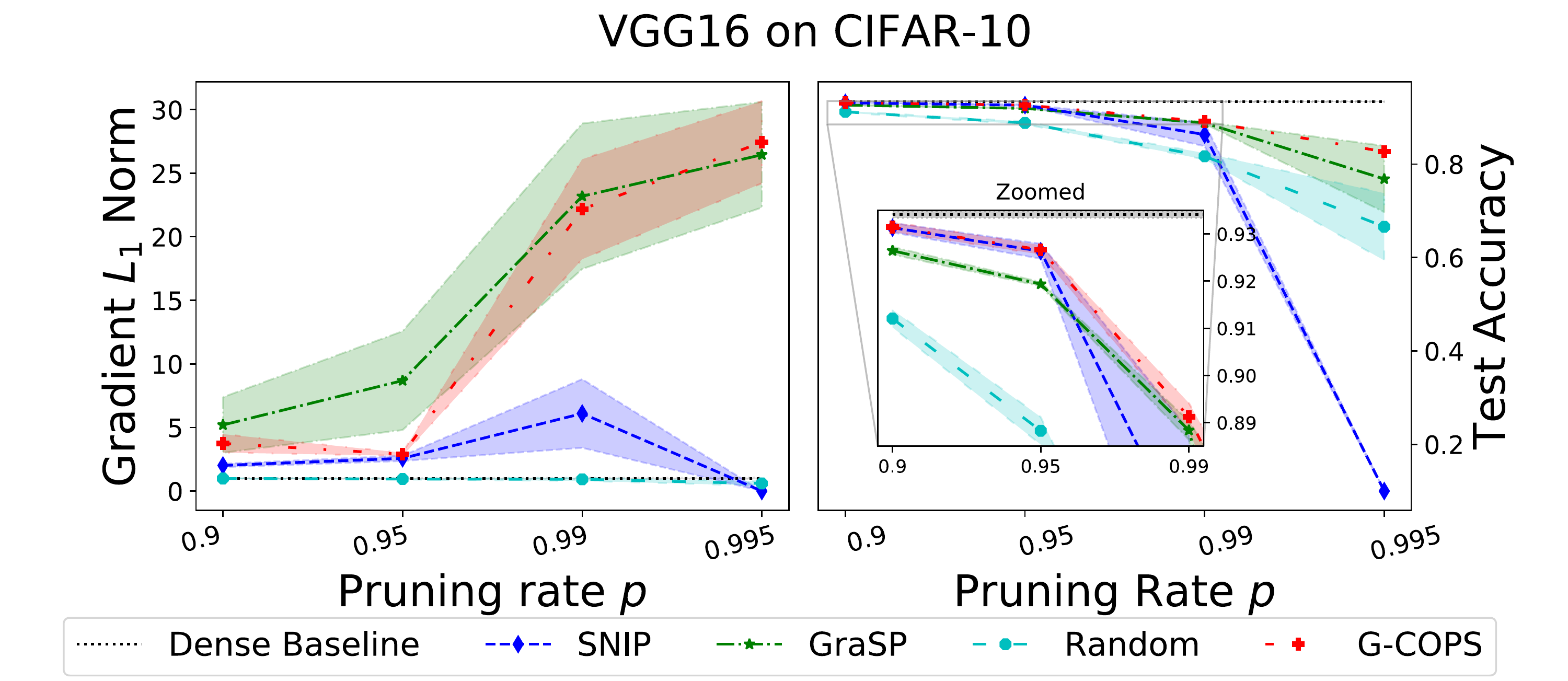}
		\caption{\label{fig:grasp_vs_snip}Comparison of dense baseline, random pruning, SNIP, GraSP and G-COPS for a VGG$16$ on CIFAR-$10$. Left: Mean gradient norm of the network's remaining weights, evaluated before training and normalized \wrt the dense network's gradient. Right: Corresponding test accuracies, see also Table \ref{tab:results_cifar10}. }
		
	\end{figure}
	\begin{table*}[htb]
		\definecolor{c1}{gray}{0.925}
		\definecolor{c2}{gray}{0.975}
		\newcolumntype{R}[1]{>{\raggedleft\let\newline\\\arraybackslash\hspace{0pt}}m{#1}}
		\centering
		\caption{Mathematical notation used in this work.}\label{tab:mathematics}
		%\begin{tabular}{|P{0.5\linewidth} p{0.4\linewidth}|}
		\begin{tabular}{>{\raggedleft\arraybackslash}p{0.399\linewidth} p{0.38\linewidth}}
			%\hline 
			\rowcolor{c1}
			$i$-th Unit Vector & {\scriptsize{}$e^{(i)}:=(\delta_{i,j})_{j=1}^{d}\in\r^{d}$ with $\delta_{i,j}=1$
				if $i=j$, else $\delta_{i,j}=0$}\tabularnewline
			\rowcolor{c2}
			$i$-th Vector Component & {\scriptsize{}$x_{i}:=\langle x,e^{(i)}\rangle:=x^{T}\cdot e^{(i)}$ for
				$x\in\r^{d},i\in\{1,\ldots,d\}$ }\tabularnewline
			\rowcolor{c1}
			Vector Support & {\scriptsize{}$\support x:=\{i\in\{1,\ldots,d\}:x_{i}\neq0\}$ for
				$x\in\r^{d}$}\tabularnewline
			\rowcolor{c2}
			$q$-Norm & {\scriptsize{}$\Vert x\Vert_{0}:=\#\support x$ and $\Vert\cdot\Vert_{q}$
				$q$-norm for $q\in(0,\infty]$}\tabularnewline
			\rowcolor{c1}
			%Closed $q$-Ball around zero & {\scriptsize{}$\ball rq:=\{y\in\r^{d}:\Vert y\Vert_{q}\leq r\}$ for
			%$q\in[0,\infty]$, $r>0$}\tabularnewline
			Pruning rate $p$ and Sparsity $\s$  & {\scriptsize{}$p:=1-d^{-1}\cdot\Vert x\Vert_{0}$ and $\s = d \cdot (1-p) = \Vert x \Vert_{0}$ for $x\in\r^{d}$}\tabularnewline
			\rowcolor{c2}
			Interior $\interior A$ of $A\subset\r^{d}$ & {\scriptsize{}$\interior A:=\{x\in A:\exists\varepsilon>0\text{ with } \{y\in\r^{d}:\Vert y - x \Vert_{2}< \varepsilon\} \subset A\}$}\tabularnewline
			\rowcolor{c1}
			$n-\argmin\{\alpha_{i}:i\in I\}$ for $I$ finite index set, $\alpha_{i}\in\r$, $n \in \mathbb{N}$ & {\scriptsize{}$\{J\subset I : \#J=\min\{\#I,n\} \text{ and } \alpha_{j}\leq\alpha_{i}\, \forall j\in J,i\in I\setminus J\}$ }\tabularnewline
			%\hline 
		\end{tabular}
	\end{table*}
	The two state-of-the-art (SOTA) pruning methods applied one-shot before
	training are \textit{Single-Shot Network Pruning based on Connection Sensitivity}
	(SNIP) \cite{lee_2018} and \textit{Gradient Signal Preservation}
	(GraSP) \cite{wang_2020}. For some conditions SNIP achieves better results
	than GraSP, for others GraSP outperforms SNIP %. But
	%usually, they do not work equally well in the same setup
	as shown
	for a VGG$16$ \cite{simonyan_2014} trained on CIFAR-$10$ \cite{krizhevsky_2012} in \figurename{} \ref{fig:grasp_vs_snip}, right-hand side.
	SNIP trains those weights having the biggest individual influence in
	changing the loss function at the beginning of training \cite{lee_2018},
	the weights with the highest \emph{saliency}. However, choosing weights only based on their solitary high saliency does not guarantee a sufficient information flow in the sparse network \cite{wang_2020}.
	For high pruning rates, this usually leads to a reduced gradient flow, visible in
	\figurename{} \ref{fig:grasp_vs_snip} for the pruning rate $p=0.995$, and finally to the pruning of whole layers \cite{tanaka_2020}. Here and in the following, gradient flow denotes the strength of the gradient signal.  Overcoming the low
	gradient flow for small pruning rates was a motivation for the GraSP
	method. %The GraSP score approximates the impact
	%of the removal of an individual weight on the network's gradient flow, called the weight's \emph{importance}, and removes those weights with the highest non-positive impa.
	GraSP zeroes those weights having the smallest \emph{importance}, an approximation of the impact of a weight's removal on the sparse network's gradient flow. Setting a weight with positive importance to zero most likely decreases the gradient flow. 
	Thus, GraSP pruned networks have a sufficiently strong gradient signal
	also for high pruning rates, as shown in \figurename{} \ref{fig:grasp_vs_snip}.
	But for lower pruning rates, where SNIP's gradient flow is strong enough, SNIP leads to better results than GraSP
	as it explicitly models the effect of pruning %on the loss function,
	%thus 
	on the networks ability to optimize the loss function. Summarized,
	a higher gradient flow does not necessarily induce a better performance. %,
	%see $p=0.9$ and $0.95$. 
	On the other hand, a gradient flow larger than zero
	is a necessary condition to train sparse networks successfully. Combining
	both, SNIP's focus on the weights' saliencies together with a strong information flow, guaranteed by GraSP, motivated us to merge
	them to achieve better results for both, high and low pruning rates. 
	In this work, all COPS computations involve a SNIP score and a second pruning score that might change between experiments. Thus, the COPS combination of SNIP and GraSP is called G-COPS. \figurename{} \ref{fig:grasp_vs_snip} compares pruning with G-COPS against SNIP and GraSP. 
	For all sparsities, G-COPS performs better than both SOTA methods SNIP and GraSP. Particularly for high pruning rates, training weights with high saliencies while guaranteeing a strong gradient flow via G-COPS improves GraSP and SNIP considerably. %if balanced properly. %via $\ak$. Bigger $\ak$ correspond to higher influence of GraSP on the final G-COPS score compared to SNIP, and vice versa.
	%For $p=0.995$, G-COPS with $\ak=0.9999$\footnote{Where $\ak$ denotes the influence of %the GraSP score, as introduced
	%	in Sections \ref{sec:Combined-Pruning-Score} and \ref{sec:Experiments}. Higher $\ak$ correspond to higher %influence of GraSP.} improves GraSP's accuracy
	%by $5.85\%$ while training only $0.8\%$ different, highly salient, parameters than
	%GraSP.%, which all have high saliencies. %The results obtained by COPS are more stable than their SNIP/GraSP
	%counterpart, shown by thinner standard deviation bands.

	\subsection{Main contributions}
	
	Our main contributions presented in this work are: 
	\setlist{nolistsep}
	\begin{itemize}[noitemsep]
		%\item Combining arbitrary and arbitrarily many generalized synaptic score based pruning methods via COPS.
		\item Combining arbitrary generalized synaptic score based pruning methods via COPS.
		\item Solving the resulting constrained, combinatorial optimization problem analytically by relaxing it on a LP.
		\item Providing an algorithm for solving the relaxed LP with lower complexity than the best known general LP solver \cite{jiang_2020}.
		\item Generating one-shot better performing sparse architectures for training than the two SOTA methods SNIP and GraSP by combining and balancing them properly. 
		\item Thereby, COPS is shown to be more efficient than naively combining two pruning scores linearly.
	\end{itemize}
	Figures in this work are best viewed in the colored online version. Table \ref{tab:mathematics} summarizes the mathematical notation. 
	%-------------------------------------------------------------------------
	\section{Related work}
	
	\paragraph{Model compression}
	
	can be achieved by methods such as \emph{quantization}, \emph{weight sharing},
	\emph{tensor decomposition}, \emph{low rank tensor approximation} or \emph{pruning}.
	Quantization reduces the number of bits used to represent the network's
	weights and/or activation maps \cite{han_2015}.
	$32$bit floats %utilized for standard deep learning 
	are replaced by low precision
	integers, thus decreasing memory consumption and speeding
	up inference. Memory reduction and speed up can also be achieved by
	weight sharing \cite{chen_2015}, tensor decomposition
	\cite{xue_2013} or low rank tensor approximation \cite{sainath_2013}
	to name only a few. 
	
	\paragraph{Pruning}
	
	is generally distinguished between \emph{structured} and \emph{unstructured}
	pruning \cite{blalock_2020}. Structured pruning deletes whole channels, neurons or even coarser
	structures, immediately
	resulting in reduced computation time. Unstructured pruning zeroes
	weights individually. Therefore, better results can be achieved and
	pruning to higher sparsity levels is possible, compared to structured pruning \cite{li_2016}.
	Setting single weights to zero does not automatically lead
	to a decreased number of computations.
	Specialized soft- and hardware \cite{han_2016} is needed
	to obtain also benefits in computational time.
	In this work, we evaluate COPS only on unstructured pruning methods
	as they often serve as foundations for corresponding structured methods, see for example \cite{li_2016}.
	But the theory derived in this paper also works for structured pruning, based on score functions, without
	the need of any further modification. 
	
	Pruned architectures can be created for instance by penalizing non-zero
	weights during training \cite{chauvin_1989}, magnitude
	pruning \cite{han_2015,frankle_2018} or saliency
	based pruning. For the latter, the significance of weights
	is measured with the Hessian of the loss \cite{lecun_1990} or the
	sensitivity of the loss with respect to inclusion/exclusion of each
	weight \cite{mozer_1989}.
	
	\paragraph{Training sparse networks}
	
	successfully from scratch was demonstrated by the Lottery Ticket Hypothesis \cite{frankle_2018}. The so trained networks, called \emph{winning tickets},
	can reach the same performance as the baseline architecture up to
	a high sparsity regime. But to find these winning tickets, many iterative
	pre-training and pruning steps are needed \cite{frankle_2018}.
	Well trainable sparse networks can also be found without costly pre-training via ranking saliencies of weights \cite{lee_2018,lee_2019}
	or preserving the dense network's information flow %before training 
	for the sparse architecture \cite{tanaka_2020,wang_2020}.
	This is done either one-shot \cite{lee_2018,wang_2020,lee_2019} or
	iteratively \cite{tanaka_2020,verdenius_2020}. \emph{Dynamic sparse training}
	\cite{bellec_2018} trains
	sparse networks but enables
	the sparse architectures to change during training. In this work, we
	focus on pruning methods applied one-shot before training starts.
	But for other scoring based pruning methods needing iterative pruning
	steps \cite{frankle_2018,tanaka_2020,verdenius_2020},
	or which are applied later on in training \cite{han_2015,lecun_1990,mozer_1989},
	our method can also be used without modifications. 
	
	\paragraph{Linear programming}
	
	In order to combine pruning scores, a combinatorial, constrained optimization
	problem is solved by %. In the spirit of Compressed Sensing \cite{candes_2005,cohen_2009,donoho_2006},
	relaxing the $\Vert\cdot\Vert_{0}$-``norm'' to the $\Vert\cdot\Vert_{1}$-norm. %Contrarily to Compressed Sensing, we relax
	%the set to optimize, not the target function. 
	The resulting relaxed problem is shown to be a LP. The dual problem
	of the relaxed problem is solved analytically with help
	of convex optimization methods \cite{boyd_2004}. In practice, contrarily to standard
	Simplex methods \cite{dantzig_1990}, we obtain the solution not by walking
	between vertices of the polytope, but by using the
	simple nevertheless robust bisection of intervals. Thus, no pivoting rules
	are needed to overcome worst case scenarios.
	LPs can also be solved fast and robustly with \emph{interior point methods} \cite{jiang_2020}. %These procedures avoid worst case scenarios by walking on paths in the polytope's interior.
	
	%-------------------------------------------------------------------------
	
	\section{Combined pruning score\label{sec:Combined-Pruning-Score}}
	
	In this section we introduce the COPS mask. The COPS mask is defined as a solution of a constrained optimization problem which optimizes
	the \emph{target score} function $S_{0}$ over \emph{pruning masks} $m\in \{0,1\}^D$, respecting a sparsity constraint $\Vert m \Vert_0 \leq \s$, while being controlled by a constraint $\k$
	on the \emph{control score} function $S_{1}$. %The general problem is proposed in Section \ref{subsec:problem}. The optimization problem is
	%solved with methods from convex optimization in Section \ref{subsec:Solution-to}. Further on, we discuss
	%the impact of the constraint $\kappa$ on the derived network architecture, Section \ref{subsec:How-to-Choose}.
	%Finally in Section \ref{subsec:Algorithm-to-Solve}, we propose an
	%end-to-end algorithm solving COPS and analyse its complexity in Section \ref{subsec:converg_compl}. The
	%mathematical results in Section \ref{subsec:Solution-to} are proven
	%for the general case of $N$ control score functions $S_{1},\ldots,S_{N}$
	%in the Supplementary Materials.
	%It can be shown that Theorem \ref{thm:problem_solved} also holds for an arbitrary number of control score functions $S_1,\ldots,S_N$ and COPS therefore might combine even more than two pruning scores.
	
	To be consistent with standard notation in convex optimization literature \cite{boyd_2004}, pruning scores are \emph{minimized} in this work. For instance, a high saliency/importance corresponds to a low SNIP/GraSP score, respectively. This is achieved by taking the negative of the original SNIP/GraSP score.
	
	\subsection{Basic assumptions and problem formulation\label{subsec:Problem-Formulation}}\label{subsec:problem}
	
	Let $f_{\Theta}$ be a DNN with vectorized weights $\Theta\in\r^{D}$.
	Pruning can be modelled by superimposing a pruning mask $m\in\{0,1\}^{D}$
	over the weights via $\mt = (m_i \cdot \Theta_{i})_{i=1}^D$. Here, $\odot$ denotes the \emph{Hadamard product}. %, %$u\odot v := (u_i\cdot v_i)_{i=1}^d
	%$ for $u,v \in \r^d$. 
	If a component $m_{i}$ of the pruning
	mask is equal to zero, the corresponding weight $\Theta_{i}$ will be pruned.
	If $m_{i}=1$, the weight $\Theta_{i}$ will be active.

	SOTA pruning methods applied without any pre-training
	use, up to changed signs, $S(m)=-\langle\left\vert \frac{\partial L}{\partial\Theta}\odot\Theta\right\vert ,m\rangle$
	(SNIP), or $S(m)=-\langle\left(\frac{\partial^{2}L}{\partial\Theta^{2}}\cdot\frac{\partial L}{\partial\Theta}\right)\odot\Theta,m\rangle$
	(GraSP) as scores. More general in \cite{tanaka_2020}, so called \emph{synaptic scores} $S(m)=\langle\frac{\partial R}{\partial\Theta}\odot\Theta,m\rangle$ are introduced. Here, $R$
	is a function depending on the weights $\Theta$ which does not need
	to be the network's loss function $L$. %For example, using $R=\sum_{i=1}^{D}\Theta_{i}^{2}$
	%leads to an equivalent formulation of magnitude based pruning. 
	The
	derivatives of $L$ and $R$ are approximated, if necessary, by a sufficient
	number of training data \cite{lee_2018, tanaka_2020, wang_2020}. For all three methods, 
	the \emph{score} $S(m)$ should indicate the performance of the pruned network $f_{m\odot\Theta}$, concerning a given criterion. Performance criteria are for example the network's gradient flow for GraSP or the ability to change the loss function for SNIP. All scores above are obtained by evaluating a linear score function $S$ on a pruning mask $m$. We call such scores \emph{generalized synaptic scores} (GSS) and $S$ a \emph{GSS function}. In the following, we are only interested in the evaluation $m\mapsto S(m)$. Thus, we assume the score function $S$ to be known and ignore potential dependencies, such as from $\Theta$ or $L$, in the notation.
	%\begin{defn}\label{def:gss}
	%[Generalized Synaptic Scores]
	%A \emph{generalized synaptic score function} is defined as a linear
	%function $S:\r^{D}\rightarrow\r,m\mapsto S(m)$. For $m \in \{0,1\}^D$, the score %$S(m)$ is called \emph{generalized synaptic score} (GSS).
	%\end{defn}
	For a GSS it holds $S(m)=\sum_{i=1}^{D}m_{i}\cdot S(e^{(i)})$
	by linearity. 
	If $\Theta_{i}$ is active, the score
	$S(e^{(i)})$ can be seen as the contribution of weight $\Theta_{i}$
	to the overall score of the network. As $S(m)$ is minimized in the following, pruning a weight $\Theta_{i}$ with high contribution $S(e^{(i)})$ is assumed to lead to better results than pruning a weight with small contribution.
	Therefore, the goal for pruning to sparsity $\s \in \SD$ with a single score function $S_0$ %, as done up to changed signs in \cite{lee_2018, tanaka_2020, wang_2020}, 
	is given by the optimization problem %to find a pruning
	%mask $\ma \in\r^{D}$ fulfilling a sparsity constraint $\Vert \ma \Vert_0 \leq \s$, $\s \in \SD$, and minimizing
	%the score function $S_{0}$. This leads to the optimization problem
	\begin{IEEEeqnarray}{c}
		\min_{m\in\X}  S_{0}(m)\;,\label{eq:unconstraint_generic} 
	\end{IEEEeqnarray}
	with 
	\begin{IEEEeqnarray}{c}
		\X:=\{m \in  \{0,1\}^D: \Vert m \Vert_0 \leq \s \}\;.
	\end{IEEEeqnarray}
	For a GSS function $S_0$, \eqref{eq:unconstraint_generic} is solved by $\ma \in \{0,1\}^D$ with
	%\begin{IEEEeqnarray}{c}
	%\ma=(\ma_{i})_{i=1}^{D}
	%\;\text{with}\;\ma_{i}=\begin{cases}
	%1 & ,\;i\in I^{\ast}\\
	%0 & ,\;else
	%\%end{cases}\;,\label{eq:m_ast_single_score}
	%\end{IEEEeqnarray}
	%\ie $\ma \in \{0,1\}^D$ and $\support \ma = I^\ast$, with
	%\begin{IEEEeqnarray}{c}
	%I^{\ast} \in \s-\argmin\{S_{0}(e^{(i)}):S_{0}(e^{(i)})<0\}\;.\label{eq:i_ast_single_score}
	%\end{IEEEeqnarray}
	\begin{IEEEeqnarray}{c}
		\support \ma \in \s-\argmin\{i \in \SD : S_{0}(e^{(i)}) < 0\} \;. \IEEEeqnarraynumspace
	\end{IEEEeqnarray} 
	As shown in \figurename{} \ref{fig:grasp_vs_snip},
	optimizing a single score does not provide the best results for all
	situations. By knowing the weakness of a score function,
	as the potential of a small gradient flow for SNIP, we can control it by
	constraining it. This is modelled by a constraint $\k \in \r$ on the control score function $S_{1}$.
	The resulting problem for GSS functions $S_{0}$ and $S_{1}$ is given by
	\begin{IEEEeqnarray}{c}
		\min_{m\in\X}S_{0}(m)\;,\;\text{such that}\;S_{1}(m)\leq\kappa\;.\ztag{$\mathcal{P}_0$} \label{eq:abstract_prob}
	\end{IEEEeqnarray} 
	In this form, since $\X$ is a discrete set, \eqref{eq:abstract_prob} is a combinatorial problem and cannot be solved with convex optimization
	methods. %In Section
	%\ref{subsec:Solution-to}, \eqref{eq:abstract_prob} is solved by
	Relaxing $\X$ on the convex optimization set 
	\begin{IEEEeqnarray}{c}
		\xr=\{m \in [0,1]^D : \Vert m \Vert_1 \leq \s\} \supset \X
	\end{IEEEeqnarray}
	%which
	%can be treated with convex optimization. 
	leads to the relaxed, linear problem
	\begin{IEEEeqnarray}{c}
		\min_{m\in\xr}S_{0}(m)\;,\text{\;such that}\;S_{1}(m)\leq\kappa\;.\ztag{$\mathcal{P}_1$}\label{eq:abstract_prob_relaxed}
	\end{IEEEeqnarray}

	\subsection{Solution to \eqref{eq:abstract_prob_relaxed} solves \eqref{eq:abstract_prob}\label{subsec:Solution-to}}
	
	As $\xr=\{m \in [0,1]^D : \Vert m \Vert_1 \leq \s\}$ and the constraint on $S_{1}$ are given by linear inequalities,
	we could solve \eqref{eq:abstract_prob_relaxed} directly with a LP solver \cite{jiang_2020,dantzig_1990}. %\footnote{The set $\xr$ can be described by $2D+1$ linear inequalities
	%for $0\leq m_{i}\leq1$ and $\Vert m \Vert_1 = \sum m_{i}\leq\s$. Because $m_{i}\geq0$
	%is given for a LP in standard form, $D+2$ constraints remain, including
	%$S_{1}(m)\leq\k$.} But in order to obtain a closed form solution for $\ma$ and to understand
	%\eqref{eq:abstract_prob_relaxed} better, 
	But in order to show that \eqref{eq:abstract_prob_relaxed} yields a solution to \eqref{eq:abstract_prob}, we will solve the corresponding dual problem. %, \ie the Lagrangian. 
	The Lagrangian
	of \eqref{eq:abstract_prob_relaxed} is given by \cite{boyd_2004}
	\begin{IEEEeqnarray}{c}
		\Lambda(m,\l):=S_{0}(m)+\l\cdot\left(S_{1}(m)-\k\right)\;\label{eq:abstract_lagrangian}
	\end{IEEEeqnarray}
	for $\l \in \r$. With the help of the Lagrangian \eqref{eq:abstract_lagrangian},
	the dual function of problem \eqref{eq:abstract_prob_relaxed} can
	be defined as 
	\begin{IEEEeqnarray}{c}
		g(\l):=\inf_{m\in\xr}\Lambda(m,\l)\;.\label{eq:abstract_abstract_dual_fct}
	\end{IEEEeqnarray}
	With $m \in \xr$ and the linearity of $S_0$ and $S_1$, 
	\begin{IEEEeqnarray}{c}\label{eq:g_lem_n1}
		g(\l)=\sum_{i\in\il}s_{i}^{(0)}+\l\cdot s_{i}^{(1)} - \l\cdot\kappa\;,\;s_{i}^{(k)}:=S_{k}\left(e^{(i)}\right)\;,
	\end{IEEEeqnarray}
	holds for every $\l \geq 0$ and arbitrary 
	\begin{IEEEeqnarray}{c}\label{eq:imin}
		\il \in \s-\argmin\{i \in \SD : s_{i}^{(0)}+\l\cdot s_{i}^{(1)} < 0\}\;.\IEEEeqnarraynumspace	
	\end{IEEEeqnarray}
	The dual function $g$ is needed to solve the dual problem
	\begin{IEEEeqnarray}{c}
		\sup_{\l\geq0}g(\l)\;.\ztag{$\mathcal{D}_1$}\label{eq:abstract_dual_prob}
	\end{IEEEeqnarray}
	%The following Lemma \ref{lem:g} and Theorem \ref{thm:problem_solved}
	%are proven in the Supplementary Material for the general case of $N$
	%control score functions $S_{1},\ldots,S_{N}$. %Additionally, a detailed
	%introduction into the dual function and the dual problem is shown
	%there. 
	\begin{thm}
		\label{thm:problem_solved}Let $\s \in \SD$, $S_0$ and $S_1$ be GSS.
		Further, assume $\k$ is chosen such that Slater's condition holds, meaning a $m\in\interior{\xr}$
		with $S_{1}(m)<\k$ exists. Then the dual problem \eqref{eq:abstract_dual_prob}
		\wrt \eqref{eq:abstract_prob_relaxed} is solved by a $\la \geq 0$. Let $\la$ be an optimal point of the dual function $g$ and $\ila$ corresponding
		indices of active scores, computed according to \eqref{eq:imin}. 
		Then
		\begin{IEEEeqnarray}{c}\label{eq:solution_primal}
			\ma \in \{0,1\}^D \; \text{with} \; \support \ma = \ila
		\end{IEEEeqnarray}
		solves the primal problem \eqref{eq:abstract_prob_relaxed}. 
		Finally, it holds $\ma\in\X\subset\xr$ and
		therefore $\ma$ is also a solution to the un-relaxed problem \eqref{eq:abstract_prob}.
	\end{thm} 
	\begin{proof}
		By \eqref{eq:g_lem_n1} and \eqref{eq:imin}, $g$ is continuous. For $\l$ big enough, Slater's condition implies that $g$ is strictly monotonically decreasing. This implicates the existence of a solution $\la \geq 0$ for \eqref{eq:abstract_dual_prob}. Strong duality induces complementary slackness (a) \cite{boyd_2004}. Therefore, for $\ma$ defined as in \eqref{eq:solution_primal}, $\ma \in \xr$ and
		\begin{IEEEeqnarray}{c}
			g(\la) \stackrel{\text{(a)}}{=} \sum_{i\in\ila}s_i^{(0)} = S_0(\ma)
		\end{IEEEeqnarray}
		is true. Meaning that $\ma$ solves \eqref{eq:abstract_prob_relaxed} by strong duality \cite{boyd_2004}. Finally, $\X \subset \xr$ and $\ma \in \X$ thus solves \eqref{eq:abstract_prob}.
	\end{proof}
	
	\subsection{How to choose the constraint $\kappa$\label{subsec:How-to-Choose}}
	\begin{figure}
		\includegraphics[width=\linewidth]{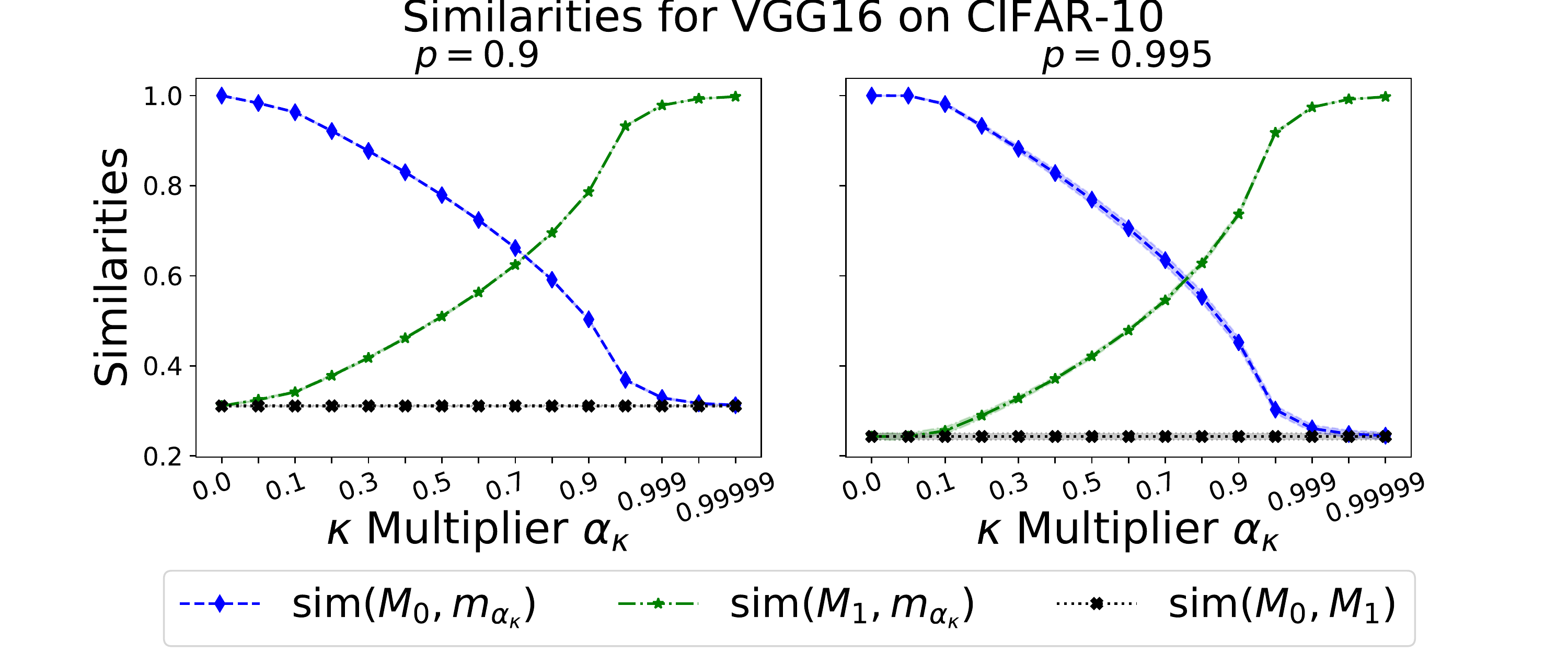}
		
		\caption{\label{fig:mask_sim}Similarity between masks derived by
			G-COPS and the underlying target score SNIP ($\simi(M_{0},m_{\ak})$)
			and control score GraSP ($\simi(M_{1},m_{\ak})$), respectively. SNIP and GraSP are compared by $\simi(M_{0},M_{1})$.
		}
	\end{figure}
	Before we present an
	algorithm solving the dual problem \eqref{eq:abstract_dual_prob}, thus by Theorem \ref{thm:problem_solved} also \eqref{eq:abstract_prob}, numerically, we need to
	discuss the constraint parameter $\kappa$. From now on we assume $\k < 0$, as $\k$ is chosen negative in all experiments considered in Section \ref{sec:Experiments}. Theorem \ref{thm:problem_solved}
	suggests using $\kappa$ such that Slater's condition holds. With $s_i^{(k)} = S_k(e^{(i)})$ and the linearity of $S_k$, $k=0,1$, the constraint $\kappa$ consequently should
	be chosen as 
	\begin{IEEEeqnarray}{c}
		\kappa>\kappa_{\min}:=\sum_{i\in I}s_{i}^{(1)}\;,\; I \in \s-\argmin\{i:s_{i}^{(1)} < 0\}\;.\IEEEeqnarraynumspace\label{eq:kappa_max}
	\end{IEEEeqnarray}
	Obviously, there is no choice of a pruned network with at most $\s$ remaining
	weights having a total $S_{1}$ score lower than $\kappa_{\min}$.
	If $\kappa$ is chosen close to $\kappa_{\min}$, the solution to \eqref{eq:abstract_prob} will be
	restricted strictly towards fulfilling the constraint. Therefore, the resulting
	network will be close to the one obtained by only taking the control scores
	$s_{1}^{(1)},\ldots,s_{D}^{(1)}$ into account. Using a constraint
	$\kappa\gg\kappa_{\min}$ on the other hand leads to less control.
	Accordingly, minimizing the target score function $S_0$ becomes
	more important. Choosing $\kappa\geq\kappa_{\max}$ with 
	\begin{IEEEeqnarray}{c}
		\kappa_{\max}:= \max_{I \in \s-\argmin\{i:s_{i}^{(0)}<0\}} \sum_{i\in I}s_{i}^{(1)}\label{eq:kappa_max_former_min}
	\end{IEEEeqnarray}
	leads to the same solution as the unconstrained problem \eqref{eq:unconstraint_generic}, considering
	only scores $s_{1}^{(0)},\ldots,s_{D}^{(0)}$. 
	
	\figurename{} \ref{fig:mask_sim} shows the similarity between the SNIP mask
	$M_{0}$, the GraSP mask $M_{1}$ and the combined
	G-COPS masks $m_{\alpha_{\k}}$ calculated for varying values for the
	constraint $\kappa:=\alpha_{\k}\cdot\kappa_{\min}$ for a VGG$16$ on CIFAR-$10$. %More similarity plots are provided in the Supplementary Materials. 
	The G-COPS masks
	are obtained by using SNIP as target and GraSP as control score. 
	As $\k_{\min}< 0$ holds, $\k_{\min}<\ak\cdot\k_{\min}$ is true for
	all $\ak\in[0,1)$. The similarity $\simi (m, M)$ between two pruning masks $m,M\in\{0,1\}^D\setminus\{0\}$
	is defined as 
	\begin{IEEEeqnarray}{c}
		\simi(m,M):=\frac{\Vert m\odot M \Vert_{0}}{\max\{\Vert m\Vert_{0},\Vert M \Vert_{0}\}} \in[0,1] \;. \label{eq:simi}
	\end{IEEEeqnarray}
	It is zero if and only if $\support m\cap\support M=\emptyset$
	and one only for $\support m=\support M$. As suspected, \figurename{} \ref{fig:mask_sim}
	shows that $m_{\alpha_{\k}}\to M_{1}$ for $\alpha_{\k}\to1$. Also,
	for small values of $\alpha_{\k}$ it holds $m_{\alpha_{\k}}\to M_{0}$.
	In between, an increasing of the similarity $\simi(M_{1-i},m_{\ak})$
	implies a decreasing of $\simi(M_{i},m_{\ak})$, $i=0,1$. Therefore, $m_{\alpha_{\k}}$
	can be seen as an interpolation between the masks $M_{0}$ and $M_{1}$,
	as desired. We additionally observe $\simi(M_{0},M_{1})>0$, \ie
	some indices are chosen for both scores. By construction of the G-COPS mask, these similar weights are also chosen for $m_{\alpha_{\k}}$
	for all $\ak<1$. The similarity between SNIP and GraSP decreases
	with an increasing pruning rate (gray line in \figurename{} \ref{fig:mask_sim}
	left plot $p=0.9$ compared to right plot $p=0.995$). Thus, the most important weights for SNIP
	are less important for GraSP and vice-versa, explaining
	SNIP's loss of gradient flow for high pruning rates. 
	
	\subsection{Algorithm to compute COPS mask\label{subsec:Algorithm-to-Solve}}
	\newcommand{\Break}{\State \textbf{break} }
	\begin{algorithm}
		
		\caption{COPS}
		\label{alg:cops}
		\begin{algorithmic}[1]
			
			\Require{Sparsity constraint $\s \in \{1,\ldots,D\}$, $\alpha_\kappa \in [0, 1)$, target scores $s_1^{(0)},\ldots, s_D^{(0)}$, control scores $s_1^{(1)},\ldots, s_D^{(1)}$, precision $\delta > 0$, upper bound $\lmax > 0$ with $g^\prime (\lmax ) \leq 0$}
			
			\Statex
			\State Set constraint $\k = \ak \cdot \k_{\min}$ \Comment{$\k_{\min}$ according to \eqref{eq:kappa_max}}
			\State Calculate $I_{0}$ as in \eqref{eq:imin} \Comment{$I_0$ gives $g(0)$ and $g^\prime(0)$ via \eqref{eq:g_lem_n1}, \eqref{eq:g_prime}}
			\If{$g^\prime (0) \leq 0$}\Comment{$g$ is monotonically decreasing for $\l \geq 0$}
			\State $\la \leftarrow 0$%, $\ila \leftarrow I_0$
			\Else \Comment{$\la \geq 0$, $\lmin=0$ is lower bound}
			\State Calculate $I_{\lmax}$ \Comment{$I_{\lmax}$ gives $g({\lmax})$ and $g^\prime({\lmax})$}
			\State $\lmin^{(0)} \leftarrow 0$, $\lmax^{(0)} \leftarrow \lmax$, $k \leftarrow 0$ \Comment{$k$ counts steps}
			
			\Repeat \Comment{Bisection}
			\State $\l_0 \leftarrow \frac{\lmin^{(k)} + \lmax^{(k)}}{2}$\Comment{Midpoint $\l_0$}
			\State Calculate $I_{\l_0}$ \Comment{$I_{\l_0}$ gives $g({\l_0})$ and $g^\prime({\l_0})$}
			\If{$g^\prime (\l_0) = 0$}\Comment{$g(\l_0)$ is optimal}
			\State $\la = \l_0$
			\Break
			\ElsIf{$g^\prime (\l_0) > 0$} \Comment{$\la \geq \l_0$}
			\State $\lmin^{(k)} \leftarrow \l_0$,  $\lmax^{(k)} \leftarrow \lmax^{(k-1)}$
			\Else \Comment{$\la \leq \l_0$}
			\State $\lmax^{(k)} \leftarrow \l_0$, $\lmin^{(k)} \leftarrow \lmin^{(k-1)}$
			\EndIf
			\State $k \leftarrow k+1$\Comment{Old $k$ now $k-1$}
			\Until{$\vert g(\lmax^{(k-1)}) - g(\lmin^{(k-1)}) \vert \leq \delta$} \Comment{Up to precision $\delta$}\label{lst:line:break}
			\State $\la \leftarrow \l_0$%, $\ila \leftarrow I_{\l_{0}}$
			\EndIf
			\State \Return Optimal mask $\ma \in \{0,1\}^D$ with $\support{\ma} = I_{\la}$% $\ma_i = 1$ if and only if $i \in I_{\la}$
		\end{algorithmic}
	\end{algorithm}
	\begin{figure}
		\centering
		\begin{tikzpicture}[scale=.75]
		\begin{scope}[scale=0.95, every node/.append style={transform shape}]
		% axes   
		\draw[->] (-.5, 0) -- (5, 0) node[above] {\scriptsize$\lambda$}; 
		\draw[->] (0, -1.5) -- (0, 1) node[above] {$y$};

		% pi_1   
		\draw[scale=1., domain=0:3, smooth, variable=\x, gray!35, dashed]  plot ({\x}, {.75 - .25 * \x});   
		\draw[scale=1., domain=3:4, smooth, variable=\x, black, dashed]  plot ({\x}, {.75 - .25 * \x});   
		\draw[scale=1., domain=4:5, smooth, variable=\x, red]  plot ({\x}, {.75 - .25 * \x});   
		\node at (-.2, .75) {\tiny{$\pi_1$}};     
		
		% pi_2   
		\draw[scale=1., domain=0:5, smooth, variable=\x, gray!35, dashed] plot ({\x}, {.375 + .125 * \x});
		\node at (-.2, .375) {\tiny{$\pi_2$}};
		%\draw[scale=1., domain=2:2.5, smooth, dotted, variable=\x, black] plot ({\x}, {.25 + .25 * \x});   
		
		% pi_3   
		\draw[scale=1., domain=0:1, smooth, variable=\x, gray!35, dashed]  plot ({\x}, {.125 -.125 * \x});   
		\draw[scale=1., domain=1:1.8, smooth, variable=\x, black, dashed]  plot ({\x}, {.125 -.125 * \x});   
		\draw[scale=1., domain=1.8:5, smooth, variable=\x, red]  plot ({\x}, {.125 -.125 * \x});   
		\node at (-.2, .125) {\tiny{$\pi_3$}};   
		%\draw[scale=1., domain=0:5, smooth, variable=\x, black, dashed]  plot ({\x}, {-.25});   
		
		% pi_4   
		\draw[scale=1., domain=0:4, smooth, variable=\x, red]  plot ({\x}, {- .25});   
		\draw[scale=1., domain=4:5, smooth, variable=\x, black, dashed]  plot ({\x}, {-.25});   
		\node at (-.2, -.25) {\tiny{$\pi_4$}};   
		
		% pi_5   
		\draw[scale=1., domain=0:1.8, smooth, variable=\x, red]  plot ({\x}, {-1. + .5 * \x});   
		\draw[scale=1., domain=1.8:2, smooth, variable=\x, black, dashed]  plot ({\x}, {-1. + .5 * \x});  
		\draw[scale=1., domain=2:4, smooth, variable=\x, gray!35, dashed]  plot ({\x}, {-1. + .5 * \x});   
		\node at (-.2, -1.) {\tiny{$\pi_5$}};   
		%\draw[scale=1., domain=3.5:3.75, smooth, dotted, variable=\x, black]  plot ({\x}, {-1. + .5 * \x});  
		
		% g(\l)   
		\draw[scale=1., domain=0:1.8, smooth, variable=\x, thick]  plot ({\x}, {-1.25 + .5 * \x});   
		\draw[scale=1., domain=1.8:4, smooth, variable=\x, thick]  plot ({\x}, {-.125 - .125 * \x});   
		\draw[scale=1., domain=4:5, smooth, variable=\x, thick]  plot ({\x}, {.875 -.375 * \x});   
		\node[thick] at (-.2, -1.375) {$g$};
		
		% show \lambda^\ast 
		\node[circle,fill=blue,inner sep=0pt,minimum size=2pt,label=below:{\color{blue}\tiny $\lambda^\ast$}] at (1.8, -.35){};  
		
		% show \max \mathcal{L}_\times    
		%\node[circle,fill=black,inner sep=0pt,minimum size=2pt,label=below:{\scalebox{.5}{$\max \mathcal{L}_{\times}$}}] at (5, -.5){};   
		
		\end{scope}
		
		\begin{scope}[xshift=.3*\textwidth, yshift=1.5cm, scale=0.95, every node/.append style={transform shape}]
		% variables 	
		\newcommand\y{1.5} 	
		\newcommand\ytwo{2.} 	
		\newcommand\xshift{2}
		
		% step 1   
		\node[] at (2.5, 0) {\tiny Step $1$};   
		\draw[scale=1., domain=0:1.8, smooth, variable=\x, ]  plot ({\x}, {-1.25 + .5 * \x});   
		\draw[scale=1., domain=1.8:4, smooth, variable=\x, ]  plot ({\x}, {-.125 - .125 * \x});   
		\draw[scale=1., domain=4:5, smooth, variable=\x, ]  plot ({\x}, {.875 -.375 * \x});   
		\node[] at (-.125, -1.25) {\scalebox{.5}{$g$}};   
		\node[circle,fill=blue,inner sep=0pt,minimum size=2pt,label=below:{\color{blue} \tiny $\lambda^\ast$}] at (1.8, -.35){};       
		\node[circle,fill=black,inner sep=0pt,minimum size=2pt,label=below:{\scalebox{.5}{$\lambda_+^{(0)}$}}] at (5, -1.){};    
		\node[circle,fill=black,inner sep=0pt,minimum size=2pt,label=below:{\scalebox{.5}{$\lambda_-^{(0)}$}}] at (0., -1.25){};    
		\node[circle,fill=black,inner sep=0pt,minimum size=2pt,label=below:{\scalebox{.5}{$\lambda_0^{(0)}\leftarrow \lambda_+^{(1)}$}}] at (2.5, -.4375) (a) {};    
		\node[below =7.5pt of a] {\scalebox{.5}{as $g^\prime (\lambda_0^{(0)}) < 0$}};

		% step 2 	
		\node[] at (1.25, - \y - .25) {\tiny Step $2$}; 	
		\draw[scale=1., domain=0:1.8, smooth, variable=\x, ]  plot ({\x}, {-1.25 - \y + .5 * \x}); 	
		\draw[scale=1., domain=1.8:2.5, smooth, variable=\x, ]  plot ({\x}, {-.125 - \y - .125 * \x}); 	
		%\draw[scale=1., domain=4:5, smooth, variable=\x, red, dotted]  plot ({\x}, {.875 -.375 * \x}); 	
		\node[] at (-.125, -1.25 - \y) {\scalebox{.5}{$g$}}; 	
		\node[circle,fill=blue,inner sep=0pt,minimum size=2pt,label=below:{\color{blue} 
			\tiny $\lambda^\ast$}] at (1.8, -.35 - \y){}; 
		\node[circle,fill=black,inner sep=0pt,minimum size=2pt,label=below:{\scalebox{.5}{$\lambda_-^{(1)}$}}] at (0., -1.25 - \y){}; 	
		\node[circle,fill=black,inner sep=0pt,minimum size=2pt,label=below:{\scalebox{.5}{$\lambda_+^{(1)}$}}] at (2.5, -.4375- \y) (a) {};    
		\node[circle,fill=black,inner sep=0pt,minimum size=2pt,label=below:{\scalebox{.5}{$\lambda_0^{(1)}\leftarrow \lambda_-^{(1)}$}}] at (1.25, -.625 - \y) (a) {}; \node[below =7.5pt of a] {\scalebox{.5}{as $g^\prime (\lambda_0^{(1)}) > 0$}};

		% step 3 	
		\node[] at (3.625, - \y - .5) {\tiny Step $3$}; 	
		\draw[scale=1., domain=1.25:1.8, smooth, variable=\x, ]  plot ({\x + \xshift}, {-1.25 - \ytwo + .5 * \x}); 	
		\draw[scale=1., domain=1.8:2.5, smooth, variable=\x, ]  plot ({\x + \xshift}, {-.125 - \ytwo - .125 * \x}); 	
		%\draw[scale=1., domain=4:5, smooth, variable=\x, red, dotted]  plot ({\x}, {.875 -.375 * \x}); 	
		\node[] at (-.125 + \xshift + 1.25, -1.25 - \ytwo + .5 * 1.25) {\scalebox{.5}{$g$}}; 	
		\node[circle,fill=blue,inner sep=0pt,minimum size=2pt,label=below:{\color{blue} \tiny $\lambda^\ast$}] at (1.8 + \xshift, -.35 - \ytwo){}; 
		
		\node[circle,fill=black,inner sep=0pt,minimum size=2pt,label=below:{\scalebox{.5}{$\lambda_-^{(2)}$}}] at (1.25 + \xshift, -.625 - \ytwo) {}; \node[circle,fill=black,inner sep=0pt,minimum size=2pt,label=below:{\scalebox{.5}{$\lambda_+^{(2)}$}}] at (2.5 + \xshift, -.4375- \ytwo) (a) {}; \node[circle,fill=black,inner sep=0pt,minimum size=2pt] at (1.875 + \xshift, -.359375 - \ytwo) (a) {}; 	
		\node[above right =-3.5pt and -1.5pt of a] (b) {\scalebox{.5}{$\lambda_0^{(3)} \leftarrow \lambda_+^{(3)}$}}; 	
		%\draw[->] (b)--(a) ; 
		
		% and so on...
		\node at (5, -\y -1){$\ldots$}; 
		\end{scope}	
		\end{tikzpicture}
		\par
		\caption{\label{fig:g_graphical}Left: Graphical example of $g$ and $\protect\la$
			with sparsity $\protect\s=2$ and underlying
			$\pi_{i}(\l) = s_i^{(0)} + \l \cdot s_i^{(1)}$. For clearness, $\kappa=0$ is used. Active $\pi_{i}(\protect\l)$ are colored {\color{red}{red}}.
			{\color{gray!35}{Gray}} coloring of $\pi_{i}(\protect\l)$ for a
			given $\protect\l$ highlights $\pi_i(\l) \geq 0$.
			Right: Shows first three steps of Algorithm \ref{alg:cops} to obtain
			a solution $\protect\la$.}
	\end{figure}
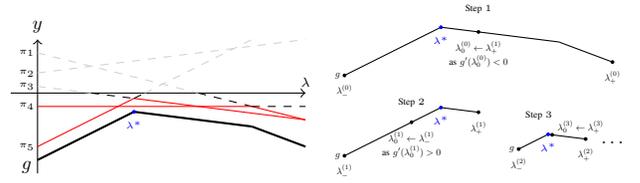
	As pointwise infimum over affine functions, $g$ is concave \cite{boyd_2004}. A so called \emph{superderivative} of $g$ can be computed for every $\l\geq 0$ by differentiating \eqref{eq:g_lem_n1} with $\il$ fixed, leading to
	\begin{IEEEeqnarray}{c}
		g^{\prime}(\l) = \sum_{i\in\il} s_i^{(1)} - \k\;.\label{eq:g_prime}
	\end{IEEEeqnarray}
	It holds by concavity of $g$ for $\l \geq 0$
	\begin{enumerate}
		\item $g^{\prime}(\l)>0$ (for one superderivative) implies $\la\geq\l$,
		\item $g^{\prime}(\l)<0$ (for one superderivative) implies $\la \leq \l$,
		%\item $g^{\prime}(\l)=0$ (for one superderivative) implies $\la = \l$,
	\end{enumerate} 
	for all solutions $\la$ of \eqref{eq:abstract_dual_prob}. Using these implications, Algorithm \ref{alg:cops} leads to a simple method for approximating $\la$ via bisection of intervals. By Theorem \ref{thm:problem_solved}, $\la$ induces a corresponding COPS mask $\ma$. 
	For applying Algorithm \ref{alg:cops}, a proper value $\lmax$ with $g^{\prime} (\lmax) \leq 0$ has to be found.
	One possibility is to start with $\lmax=1$. If $g^{\prime}(\lmax)\geq0$,
	$\lmax$ can be doubled iteratively. A graphical illustration for Algorithm \ref{alg:cops} is given in \figurename{} \ref{fig:g_graphical}. 
	
	\subsection{Convergence and complexity of Algorithm \ref{alg:cops}}\label{subsec:converg_compl}
	Algorithm \ref{alg:cops} has linear order of convergence due to the used bisection method. By \eqref{eq:imin}, $g(\l)$ and $g^\prime(\l)$ can be computed in $\mathcal{O}(D\cdot \log D)$ by sorting $\{s_i^{(0)} + \l s_i^{(1)}\}$.
	This together leads to a deterministic complexity $\mathcal{O} \left( D\cdot \log D \cdot  \log \frac{1}{\delta} \right)$ of Algorithm \ref{alg:cops} for an arbitrary sparsity $1\leq \sigma \leq D$ and precision $\delta > 0$. Precision $\delta$ guarantees $\vert g( \la_{COPS}) - g(\la) \vert \leq \delta$ for $\la_{COPS}$, the solution found by COPS, and an exact solution $\la$ of the dual problem. Using \emph{faster dynamic matrix inverses} (FDMI) \cite{jiang_2020}, the best expected runtime for solving a LP
	is given by $\mathcal{O}(D^{2.055}\cdot \log \frac{D}{\delta})$ which is best to our knowledge
	the fastest method for solving a general LP. %But different to Algorithm \ref{alg:cops}, FDMI allows arbitrary, dense constraint matrices. 
	With $D\cdot \log D \ll D^{2.055}$ we see that COPS has drastically smaller complexity than the best standard LP solver FDMI. This speed up is achieved since Algorithm \ref{alg:cops} does not invert matrices during optimization, but only sorts scores. The acceleration induced by COPS compared to using standard LP solvers is small in relation to the overall training time for the setup of one-shot pruning before training---usually pruning takes minutes whereas training lasts hours. But if iterative pruning or dynamic sparse training is used, the LP will be solved more often. Thus, COPS' speed up will be multiplied by the number of score evaluations. %Thus, using the relatively slow but stable bisection method still outperforms the LP solver FDMI. Additionally, COPS' complexity is not expected, but
	%deterministic.
	
	%-------------------------------------------------------------------------
	\section{Experiments and discussions\label{sec:Experiments}}
	\begin{table*}[htb!]
		\centering
		\caption{\label{tab:Experimental-setups-for}Experimental setups with references to used training and hyperparameter setup given in brackets in the top row.}
		\begin{tabular}{cccccc}
			{\scriptsize{}} & {\scriptsize{}MNIST (\cite{lee_2018})} & {\scriptsize{}CIFAR-$10$ (\cite{lee_2018})} & {\scriptsize{}CIFAR-$100$ (\cite{tanaka_2020})} & {\scriptsize{}Tiny ImageNet/ResNet$18$ (\cite{tanaka_2020})}& {\scriptsize{}Tiny ImageNet/WRN$18$-$2$ (\cite{tanaka_2020})}\tabularnewline
			\hline 
			{\scriptsize{}$\#$Epochs} & {\scriptsize{}$250$} & {\scriptsize{}$250$} & {\scriptsize{}$160$} & {\scriptsize{}$100$} & {\scriptsize{}$100$}\tabularnewline
			{\scriptsize{}Batch Size} & {\scriptsize{}$100$} & {\scriptsize{}$128$} & {\scriptsize{}$128$} & {\scriptsize{}$128$} & {\scriptsize{}$128$}\tabularnewline
			{\scriptsize{}Optimizer: SGD-} & {\scriptsize{}Momentum $0.9$} & {\scriptsize{}Momentum $0.9$} & {\scriptsize{}Momentum $0.9$} & {\scriptsize{}Momentum $0.9$}& {\scriptsize{}Momentum $0.9$}\tabularnewline
			{\scriptsize{}Learning Rate} & {\scriptsize{}$0.1$} & {\scriptsize{}$0.1$} & {\scriptsize{}$0.01$} & {\scriptsize{}$0.1$}& {\scriptsize{}$0.1$} \tabularnewline
			{\scriptsize{}LR Decay} & {\scriptsize{}$\times0.1$} & {\scriptsize{}$\times0.1$} & {\scriptsize{}$\times0.2$} & {\scriptsize{}$\times0.1$} & {\scriptsize{}$\times0.1$}\tabularnewline
			& {\scriptsize{}every $25$k iterations} & {\scriptsize{}every $30$k iterations} & {\scriptsize{}epochs $60/120$} & {\scriptsize{}epochs $30/60/80$} & {\scriptsize{}epochs $30/60/80$}\tabularnewline
			{\scriptsize{}Weight Decay} & {\scriptsize{}$5\cdot10^{-4}$} & {\scriptsize{}$5\cdot10^{-4}$} & {\scriptsize{}$5\cdot10^{-4}$} & {\scriptsize{}$10^{-4}$} & {\scriptsize{}$10^{-4}$}\tabularnewline
			{\scriptsize{}Initialization} & {\scriptsize{}Glorot \cite{xavier_2010}} & {\scriptsize{}He \texttt{fan\_in} \cite{he_2016}} & {\scriptsize{}He \texttt{fan\_in} \cite{he_2016}} & {\scriptsize{}He \texttt{fan\_in} \cite{he_2016}} & {\scriptsize{}He \texttt{fan\_in} \cite{he_2016}}\tabularnewline
		\end{tabular}	
	\end{table*}
	In this section we present and discuss experimental results for pruning one-shot before training via COPS. % on
	%different datasets and DNN architectures.
	To show the superior performance of combining SNIP and GraSP via G-COPS, compared to their individual performance, we evaluate them on the image classification tasks CIFAR-$10$, CIFAR-$100$ \cite{krizhevsky_2012}
	and Tiny ImageNet \cite{tiny_imagenet}. The tested DNN architectures are VGG$16$ \cite{simonyan_2014}, ResNet$18$ \cite{he_2016} and Wide-ResNet$18$\hbox{-}$2$ (WRN$18$-$2$) \cite{zagoruyko_2016}. SNIP and GraSP have shown to be competitive with iterative pruning before training and pruning methods applied during training or after pre-training \cite{lee_2018,tanaka_2020,wang_2020}. 
	To underline that COPS is not restricted to SNIP/GraSP, we further control SNIP with magnitude scores (M-COPS) and combine SNIP with an unsupervised version of SNIP (USNIP) as a control score, called U-COPS. The U-COPS experiment is conducted on a LeNet-$5$-Caffe \cite{lecun_1998}, trained on MNIST \cite{lecun_1998}.
	
	Inspired by the interpolation of G-COPS between SNIP and GraSP, shown
	in \figurename{} \ref{fig:mask_sim}, a grid-like search for the multiplier
	$\ak$ for $\k_{\min}$ is used. The considered values are $\ak\in\{0.05,0.1,0.2,\ldots,0.9\}$
	together with a more fine-grained search close to $1$ via $\ak\in\{0.99,0.999,0.9999,0.99999\}$.
	Values $\ak$ close to one become more interesting for higher pruning
	rates, as SNIP tends to generate architectures with a weak gradient
	signal. Thus, a stricter control regarding gradient flow will lead to better results in these cases. We use the same grid-like search also for M-COPS and U-COPS.  To achieve fair comparisons, reported $\ak$ in Tables \ref{tab:results_cifar10} and \ref{tab:results_bigger} are chosen as those with the best mean validation accuracy at early stopping time among all tested $\ak$. 
	
	\subsection{Experimental setup}\label{subsec:experimental_setup}
	For the experiments presented in this work we used the deep learning framework PyTorch$1.5$
	\cite{pytorch_ieee}. All experiments were run on a single Nvidia GeForce
	$1080$ti GPU.
	
	As common in the literature of pruning before training starts \cite{lee_2018,tanaka_2020,wang_2020,lee_2019,verdenius_2020},
	only weights were pruned. All biases and Batch Normalization
	parameters were kept trainable. Therefore, the pruning rate $p$ only considers weights. 
	Like other pruning methods applied before training \cite{lee_2018,tanaka_2020,wang_2020,lee_2019,verdenius_2020},
	also COPS is evaluated on a fixed training setup. Therefore, no hyperparameter optimization was executed, except for $\ak$. The training setup and hyperparameters together with references to the used network architectures are given in Table \ref{tab:Experimental-setups-for}.
	
	We performed five runs for each reported result for MNIST and CIFAR and three for Tiny ImageNet. Each run was based on a different random seed for weight initialization and ordering of the data.
	The common data augmentations for CIFAR and Tiny ImageNet were used, with code based on\footnote{\url{https://github.com/alecwangcq/GraSP}}\addtocounter{footnote}{-1}\addtocounter{Hfootnote}{-1}. The training data for MNIST and CIFAR was split randomly
	$9/1$ for training and validation for each random seed. Results are reported for the early
	stopping epoch, the epoch with the highest validation accuracy. 
	
	The considered pruning scores are i.i.d. $\mathcal{U}(-1,0)$ random scores, magnitude scores $-\vert \Theta_{i}\vert$,
	GraSP{\footnotemark} \cite{wang_2020}, SNIP\footnote{PyTorch adaptation of \url{https://github.com/namhoonlee/snip-public}}\addtocounter{footnote}{-1}\addtocounter{Hfootnote}{-1} \cite{lee_2018} and USNIP{\footnotemark} \cite{lee_2019}. The determination of SNIP and GraSP scores is discussed in Section \ref{subsec:problem}. The USNIP method will be introduced in Section \ref{subsec:ucops}. %Contrarily to
	%common notation for the calculation of pruning scores, we minimize the
	%scores to stick with standard formulations of convex optimization. 
	Except random and magnitude based pruning, all considered pruning methods are data dependent.
	To obtain a sufficient statistic of the training data, we used the data of $100$
	randomly chosen training batches to calculate the SNIP, GraSP and USNIP scores. 
	%\emph{Calculation of the pruning mask for a single score:}
	For a pruning method with a single score function, a pruning mask with sparsity $\s$ is computed as a solution to \eqref{eq:unconstraint_generic}. 
	%\emph{Computation of COPS masks:} 
	For computing the COPS mask,
	we implemented Algorithm \ref{alg:cops} with precision $\delta=10^{-6}$. To obtain scores
	in the same range, the target and control scores were normalized to have $\sum_{i=1}^D \vert s_i^{(k)} \vert = 1$, $k=0,1$.
	
	\subsection{Experimental results\label{subsec:Experimental-Results}}
	For CIFAR-$10$, we tested two different versions of COPS, G-COPS and M-COPS. The results can be seen in Table \ref{tab:results_cifar10}. Using GraSP as control score leads to similar results for moderate pruning rates, but better results for higher rates, than controlling the weights' magnitudes. Therefore, we only analyze G-COPS for the bigger datasets CIFAR-$100$ and Tiny ImageNet, and solely discuss G-COPS in the following paragraphs. Still, controlling magnitudes improves SNIP's test accuracy by more than $1.3\%$ for $p=0.99$. Remarkably, pruning solely based on magnitudes does not even supply trainable sparse networks for this pruning rate at all. 
	
	\begin{table}[tb]
		\centering	
		\caption{\label{tab:results_cifar10}Experiments on CIFAR-$10$ with VGG$16$.} 	
		\begin{tabular}{@{}lcccc@{}}
			%\toprule
			%& \multicolumn{4}{c}{CIFAR-$10$/VGG$16$}\tabularnewline
			{Method} & {\scriptsize{}$p=0.9$} & {\scriptsize{}$p=0.95$} &{\scriptsize{}$p=0.99$} & {\scriptsize{}$p=0.995$}\tabularnewline
			\hline
			Baseline & \multicolumn{4}{c}{{\scriptsize{}$93.41\pm0.07\%$}}\tabularnewline
			Random & {\scriptsize{}$91.21\%$} & {\scriptsize{}$88.83\%$} & {\scriptsize{}$81.70\%$} & {\scriptsize{}$66.59\%$}\tabularnewline
			SNIP & {\scriptsize{}$\underline{93.13\%}$} & {\scriptsize{}\underline{$92.63\%$}} & {\scriptsize{}$86.34\%$} & {\scriptsize{}$10.00\%$}\tabularnewline
			GraSP & {\scriptsize{}$92.64\%$} & {\scriptsize{}$91.94\%$} & {\scriptsize{}$\underline{88.84\%}$} & {\scriptsize{}$\underline{76.82\%}$}\tabularnewline
			Magnitude & {\scriptsize{}$92.94\%$} & {\scriptsize{}$92.16\%$} & {\scriptsize{}$10.00\%$} & {\scriptsize{}$10.00\%$}\tabularnewline
			\hline 
			G-COPS & {\scriptsize{}${\bf 93.15\%}$} & {\scriptsize{}${92.66\%}$} & {\scriptsize{}${\bf 89.12\%}$} & {\scriptsize{}${\bf 82.67\%}$}\tabularnewline
			($\ak$) & {\scriptsize{}($0.1$)} & {\scriptsize{}($0.05$)} & {\scriptsize{}($0.99999$)} & {\scriptsize{}($0.9999$)}\tabularnewline
			\hline
			M-COPS  & {\scriptsize{}${\bf 93.15\%}$} & {\scriptsize{}${\bf 92.68\%}$} & {\scriptsize{}$87.71\%$} & {\scriptsize{}$22.30\%$}\tabularnewline
			($\ak$) & {\scriptsize{}($0.1$)} & {\scriptsize{}($0.3$)} & {\scriptsize{}($0.5$)} & {\scriptsize{}($0.4$)}\tabularnewline
			%\bottomrule
		\end{tabular}
	\end{table}
	
	\begin{table*}[tb]
		\centering
		\caption{\label{tab:results_bigger}Comparison of G-COPS ($m_{\ak}$), SNIP ($M_0$) and GraSP ($M_1$); $\Delta$ denotes the accuracy difference to G-COPS' result.}
		\begin{tabular}{@{}l@{\hskip 25pt}ccc@{\hskip 25pt}ccc@{\hskip 25pt}ccc@{}}
			%\toprule
			& \multicolumn{3}{c@{\hskip 25pt}}{CIFAR-$100$/ResNet$18$} & \multicolumn{3}{c@{\hskip 25pt}}{Tiny ImageNet/ResNet$18$} & \multicolumn{3}{c@{}}{Tiny ImageNet/WRN$18$-$2$}\tabularnewline
			Method & {\scriptsize{}$p=0.85$} & {\scriptsize{}$p=0.95$} & {\scriptsize{}$p=0.99$} & {\scriptsize{}$p=0.85$} & {\scriptsize{}$p=0.95$} & {\scriptsize{}$p=0.99$} & {\scriptsize{}$p=0.85$} & {\scriptsize{}$p=0.95$} & {\scriptsize{}$p=0.99$}\tabularnewline
			\hline
			Baseline & \multicolumn{3}{c@{\hskip 25pt}}{{\scriptsize{}$74.53\pm0.33\%$}} & \multicolumn{3}{c@{\hskip 25pt}}{{\scriptsize{}$55.97\pm0.37\%$}} & \multicolumn{3}{c@{}}{{\scriptsize{}$59.38\pm0.25\%$}}\tabularnewline
			Random & {\scriptsize{}$70.71\%$} & {\scriptsize{}$64.54\%$} & {\scriptsize{}$50.70\%$} & {\scriptsize{}$52.59\%$} & {\scriptsize{}$46.96\%$} & {\scriptsize{}$33.87\%$} & {\scriptsize{}$56.52\%$} & {\scriptsize{}$52.94\%$} & {\scriptsize{}$42.63\%$}\tabularnewline
			\hline
			SNIP & {\scriptsize{}$\underline{71.61\%}$} & {\scriptsize{}$68.39\%$} & {\scriptsize{}$56.69\%$} & {\scriptsize{}$\underline{56.10\%}$} & {\scriptsize{}$\underline{54.50\%}$} & {\scriptsize{}$37.00\%$} & {\scriptsize{}$\underline{58.73\%}$} & {\scriptsize{}$\underline{57.71\%}$} & {\scriptsize{}$51.07\%$}\tabularnewline
			$\Delta$ & {\scriptsize{}${\bf -0.06\%}$} & {\scriptsize{}${\bf -0.13\%}$} & {\scriptsize{}${\bf -2.56\%}$} & {\scriptsize{}${\bf -0.06\%}$} & {\scriptsize{}${\bf-0.25\%}$} & {\scriptsize{}${\bf -7.27\%}$} & {\scriptsize{}${\bf -0.03\%}$} & {\scriptsize{}${\bf -0.06 \%}$} & {\scriptsize{}${\bf -1.94 \%}$}\tabularnewline
			$\simi(M_0, m_{\ak})$& {\scriptsize{}$0.988$} & {\scriptsize{}$0.938$} & {\scriptsize{}$0.302$} & {\scriptsize{}$0.985$} & {\scriptsize{}$0.990$} & {\scriptsize{}$0.277$} & {\scriptsize{}$0.982$} & {\scriptsize{}$0.981$} & {\scriptsize{}$0.271$}\tabularnewline
			\hline 
			GraSP & {\scriptsize{}$71.16\%$} & {\scriptsize{}$\underline{68.40\%}$} & {\scriptsize{}$\underline{58.67\%}$} & {\scriptsize{}$54.77\%$} & {\scriptsize{}$52.83\%$} & {\scriptsize{}$\underline{43.81\%}$} & {\scriptsize{}$56.87\%$} & {\scriptsize{}$57.22\%$} & {\scriptsize{}$\underline{52.69\%}$}\tabularnewline
			$\Delta$ & {\scriptsize{}${\bf -0.51\%}$} & {\scriptsize{}${\bf -0.12\%}$} & {\scriptsize{}${\bf -0.58\%}$} & {\scriptsize{}${\bf -1.39\%}$} & {\scriptsize{}${\bf-1.92\%}$} & {\scriptsize{}${\bf -0.46\%}$} & {\scriptsize{}${\bf -1.89\%}$} & {\scriptsize{}${\bf -0.55 \%}$} & {\scriptsize{}${\bf - 0.32\%}$}\tabularnewline
			$\simi(M_1, m_{\ak})$ & {\scriptsize{}$0.351$} & {\scriptsize{}$0.350$} & {\scriptsize{}$0.994$} & {\scriptsize{}$0.343$} & {\scriptsize{}$0.306$} & {\scriptsize{}$0.998$} & {\scriptsize{}$0.355$} & {\scriptsize{}$0.303$} & {\scriptsize{}$0.994$}\tabularnewline
			\hline
			G-COPS & {\scriptsize{}${\bf 71.67\%}$} & {\scriptsize{}${\bf 68.52\%}$} & {\scriptsize{}${\bf 59.25\%}$} & {\scriptsize{}${\bf 56.16\%}$} & {\scriptsize{}${\bf 54.75 \%}$} & {\scriptsize{}${\bf 44.27\%}$} & {\scriptsize{}${\bf 58.76\%}$} & {\scriptsize{}${\bf 57.77}\%$} & {\scriptsize{}${\bf 53.01\%}$}\tabularnewline
			($\ak$) & {\scriptsize{}($0.05$)} & {\scriptsize{}($0.2$)} & {\scriptsize{}($0.9999$)} & {\scriptsize{}($0.05$)} & {\scriptsize{}($0.05$)} & {\scriptsize{}($0.99999$)} & {\scriptsize{}($0.05$)} & {\scriptsize{}($0.05$)} & {\scriptsize{}($0.9999$)}\tabularnewline
			%	\bottomrule
		\end{tabular}
		
	\end{table*}
	
	Tables \ref{tab:results_cifar10} and \ref{tab:results_bigger}
	compare results for G-COPS, GraSP and SNIP for different pruning rates $p$,
	classification tasks and network architectures. 
	For every task and pruning rate, the best performing G-COPS has better accuracy than SNIP. Using more moderate pruning
	rates $p\leq 0.95$, SNIP and the best G-COPS results lie close together. The gradient flow for SNIP pruned networks is sufficient for these rates, as shown in \figurename{}~\ref{fig:grasp_vs_snip} for CIFAR-$10$. Thus, controlling SNIP's gradient flow only helps marginally. Applying SNIP
	without constraint seems to be the most economic choice for moderate
	pruning rates, since $\ak$ has not to be optimized as hyperparameter. For all models, GraSP
	and random pruning is outperformed by both, G-COPS and SNIP, for $p\leq 0.95$. %We want to highlight that we only did a grid search
	%for $\ak$. %Using more sophisticated hyperparameter searches for the scaling
	%factor $\ak$, as \cite{falkner_2018}, might improve G-COPS even further.
	
	As predicted, bigger
	$\ak$ achieve the best results for G-COPS with higher pruning rates.
	For these rates we see G-COPS' advantage of not
	being bounded to one pruning score. Using only gradient flow as a
	score (GraSP) provides better results than using only SNIP scores.
	Here, not taking the gradient flow into account might
	result in architectures which do not train at all, like SNIP for $p=0.995$ in Table \ref{tab:results_cifar10}. But, looking only at
	the gradient signal can be improved further by additionally optimizing for the
	sparse architecture with the best influence on the training loss.
	Even small changes in GraSP's architecture by a few weights with high saliencies lead to considerably
	better results, as shown by G-COPS for the ResNet$18$ on CIFAR-$100$ and Tiny ImageNet
	($p=0.99$) and the VGG$16$ on CIFAR-$10$ ($p=0.995$). For the latter,
	using G-COPS with $\ak=0.9999$, inducing a similarity of $0.992$ between
	G-COPS and GraSP, surpasses GraSP's accuracy by $5.85\%$. Therefore,
	gradient flow itself should not be the only criterion for choosing
	a sparse architecture to train. Nevertheless, sufficient gradient
	flow must be guaranteed for a successful training. As shown, G-COPS offers a possibility to combine these criteria fruitfully.
	
	\subsection{COPS is not a binary decision between two scores}
	Tables \ref{tab:results_cifar10} and \ref{tab:results_bigger} show that the grid-like search for $\ak$ tends to find a G-COPS mask close to either the SNIP or the GraSP mask. The results also show that in almost every case there is a significant performance gap between SNIP and GraSP. Thus, it seems likely that a better performing pruning mask lying in between those two will be close to the superior mask. Such a mask is always found by G-COPS and it clearly improves results compared to GraSP and SNIP, as discussed in Section \ref{subsec:Experimental-Results}. So, even if G-COPS is close to either SNIP or GraSP, it is not a binary decision between those pruning methods.  Preferably, COPS is esteemed as a method that combines the better parts of two pruning methods.
	
	\subsection{COPS pruning with partial class information}\label{subsec:ucops}
	\begin{figure}[tb]
		\centering
		
		\includegraphics[width=1\linewidth]{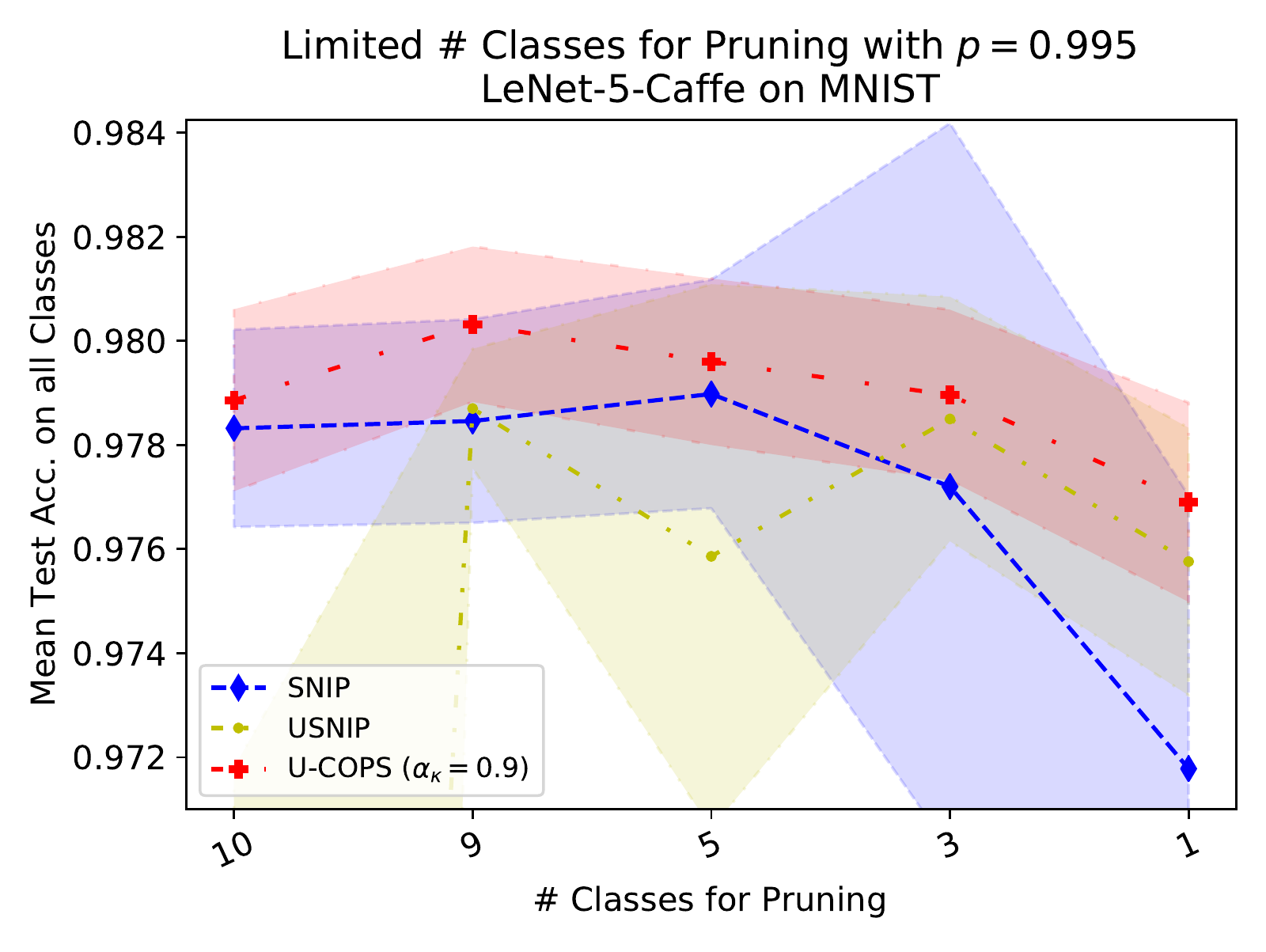}
		
		\caption{Comparison of SNIP, USNIP and U-COPS. Colored bars represent one standard deviation. Due to unstable training, USNIP has $89.8 \pm 7.4\%$ accuracy if all classes are used to compute the pruning score which is not shown in the plot for reasons of lucidity.}\label{fig:usnip2}
	\end{figure}
	The primal motivation for COPS is given by linking the pruning methods SNIP
	and GraSP. However, it is not restricted to this combination. In the
	following, we prune networks with partial class information available at the time of pruning. Meaning, $T\subset\{1,\ldots,C\}$ classes of the original $C$ classes are used
	to compute the pruning score. Thus, the classes $\{1,\ldots,C\}\setminus T$
	are not known to the pruning procedure. All classes $\{1,\ldots,C\}$ are further used to train and evaluate the pruned network. Analyzing a pruning score in such a way is important if the data used for computing the score will be an unrepresentative sample of the underlying data distribution. As baseline
	pruning procedure we again use SNIP \cite{lee_2018}. The control
	{USNIP}
	score \cite{lee_2019} is almost identical to SNIP. They only differ as the loss for USNIP is calculated with labels drawn uniformly at random instead of the real labels. By backpropagation of the loss, every
	class is included into the calculation of the USNIP score even if
	the corresponding input image will belong most probably to another class. %The COPS combined version of them is called U-COPS.
	We evaluate these methods on a small CNN, LeNet-$5$-Caffe, trained on
	MNIST. As MNIST is solvable almost perfectly while using only few
	parameters% (shown in the Supplementary Materials)
	, we chose the pruning rate
	$p=0.995$. Results for SNIP, USNIP and U-COPS are shown in \figurename{} \ref{fig:usnip2}.
	
	SNIP pruned networks reach stable results for $\#T\geq5$. %\footnote{Whereas for $\#T =10$, USNIP generates $3/5$ networks with test accuracy $< 90\%$ resulting in $89.80 \pm 7.38 \%$ (mean $\pm$ std) test accuracy.} 
	However with fewer classes used to calculate the score, SNIP's performance degrades steeply.
	For these low numbers of regarded classes, USNIP provides better results than SNIP.
	But, combining
	both methods via U-COPS reaches the best results for all treated $\#T$.
	In this experiment, 
	U-COPS with a single multiplier
	$\ak=0.9$ reaches better performances than both individual pruning
	methods for all numbers of considered classes. Also, $\ak=0.9$ provides,
	contrarily to SNIP and USNIP, stable results over all $\#T$. 
	
	This
	experiment shows that COPS is not limited to standard pruning methods like SNIP, GraSP or magnitude pruning,
	but also combines supervised and unsupervised pruning successfully. U-COPS even achieves better results than SNIP for $\#T = 10$. Therefore, pruning based on noisy labels via U-COPS, with noise level controlled by $\ak$, can generate better trainable sparse architectures. This effect for pruning is comparable to \emph{label smoothing} \cite{szegedy_2016} for standard DNN training.

	\subsection{Stability of COPS}\label{subsec:stability}
	\begin{figure}[tb]
		\centering
		
		\includegraphics[width=\linewidth]{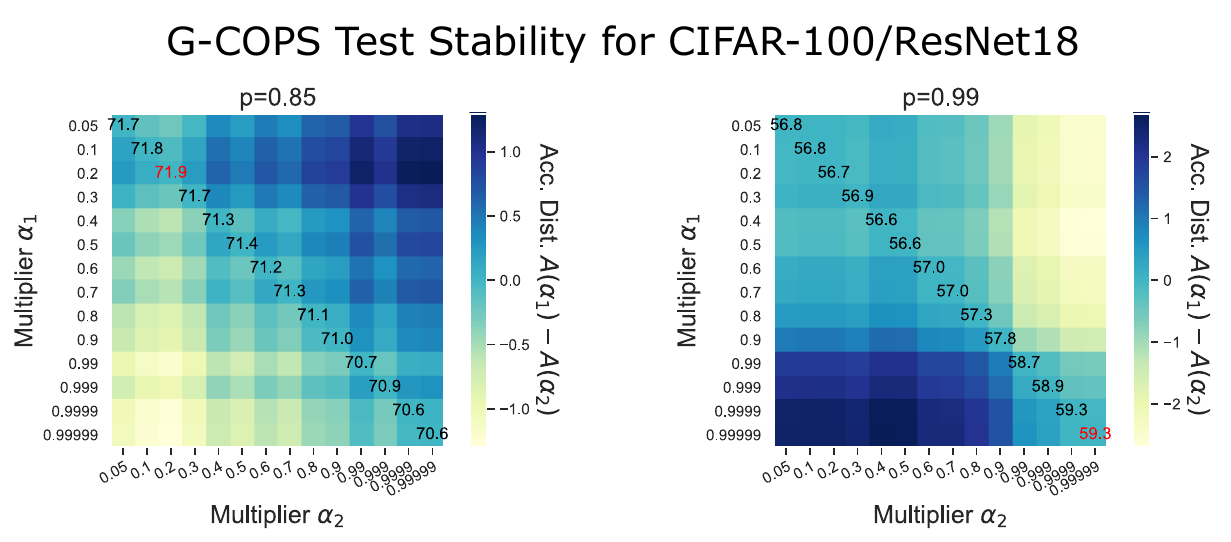}
		
		\caption{Stability analysis for G-COPS on CIFAR-$100$ with corresponding test accuracies $A(\ak)$ displayed on the diagonal. Best accuracy highlighted \color{red}{red}.}\label{fig:stability_cifar100}
	\end{figure}
	\figurename{} \ref{fig:stability_cifar100} analyzes G-COPS' stability for the CIFAR-$100$ experiment, Table \ref{tab:results_bigger}, for two pruning rates $p=0.85$ and $p=0.99$. On the heatmaps' diagonals, test accuracies for all considered constraint multipliers $\ak$ are displayed. A cell's color represents the accuracy difference of the G-COPS pruned networks with $\ak$ in corresponding row and column. As suspected, low constraints reach better results for $p=0.85$ than higher ones, whereas for $p=0.99$ this behaviour is reversed. The heatmap for $p=0.85$ shows that G-COPS is \emph{stable} for all $\ak$. Meaning, for small changes of the constraint multiplier $\ak$, the corresponding architectures also have small differences in test accuracy. But for $p=0.99$, there is an accuracy gap of $0.9\%$ between $\ak=0.9$ and $\ak=0.99$. This performance drop for $\ak \leq 0.9$ emphasizes the necessity of a strict control on the gradient flow for high pruning rates. Disregarding this accuracy gap, G-COPS also reaches stable results for $p=0.99$.
	
	The heatmap for $p=0.85$ also reveals that the best validation accuracy (Table \ref{tab:results_bigger}: $\ak=0.05$) not always provides the best test accuracy ($\ak=0.2$). 
	Still, since accuracy patterns for validation and testing are similar, G-COPS' stability guarantees the best $\ak$ for validation to be in a neighbourhood of similar test performance as the best $\ak$ for testing.
	
	\subsection{Why solving the dual problem?}
	\begin{table}[tb]
		\centering	
		\caption{\label{tab:rnd_search}Random search on CIFAR-$10$ with VGG$16$ for G-COPS ($\ak$) and $\l$ for the linear combination $S_0 + \l S_1$, $14$ hyperparameters each.} 	
		\begin{tabular}{@{}lcccc@{}}
			%\toprule
			%& \multicolumn{4}{c}{CIFAR-$10$/VGG$16$}\tabularnewline
			{Method} & {\scriptsize{}Best Test Acc.} & {\scriptsize{}Worst Test Acc.} &{\scriptsize{}Mean Test Acc.} & {\scriptsize{}Std Test Acc.}\tabularnewline
			\hline
			G-COPS & {\scriptsize{}$93.08\%$} & {\scriptsize{}$92.79\%$} & {\scriptsize{}$92.93\%$} & {\scriptsize{}$0.16\%$}\tabularnewline
			$S_0 + \l S_1$ & {\scriptsize{}$92.77\%$} & {\scriptsize{}$92.46\%$} & {\scriptsize{}$92.62\%$} & {\scriptsize{}$0.20\%$} \tabularnewline
			
			%\bottomrule
		\end{tabular}
	\end{table}
	Using the strong duality \cite{boyd_2004} of the problem solved by Algorithm \ref{alg:cops} with constraint multiplier $\ak$, we see that the COPS mask $\ma$ can equivalently be found as the solution of
	\begin{IEEEeqnarray}{c}
		\min_{m\in \xr} S_0(m) + \la S_1(m) \label{eq:simple_lc}\;,
	\end{IEEEeqnarray}
	where $\la$ solves \eqref{eq:abstract_dual_prob}. %Thus, $\ma$ is simply given by optimizing a linear combination of the two scores $S_0$ and $S_1$. 
	A natural question is why using COPS should be more convenient than combining two scores $S_0$ and $S_1$ naively via $S_0 + \l S_1$.
	%In the following, we will show that using Algorithm \ref{alg:cops} compared to solving \eqref{eq:simple_lc} is twofold. 
	First, $\ak$ is an interpretable hyperparameter, as shown in Section \ref{subsec:How-to-Choose}. An $\ak$ is closely related to the similarity of the COPS mask and its underlying masks, see \figurename{} \ref{fig:mask_sim}. It also controls the strictness of the corresponding constraint $\ak \cdot \k_{\min}$. For a specific task, this leads to a valuable prior knowledge for finding the right hyperparameter $\ak$. \Ie high pruning rates result in low gradient flow for SNIP, therefore a strict control of the GraSP score is necessary. %, as done with the fine grained search for $\ak$ close to $1$ in section \ref{subsec:Experimental-Results}. 
	Choosing $\l$ as hyperparameter on the other hand only balances the scores which gives no additional insight. Second, $(0, 1) \ni \ak$ is a much smaller interval than $(0, \infty) \ni \l$, which improves hyperparameter search. Table \ref{tab:rnd_search} shows a comparison for random search \cite{bergstra_2012} between G-COPS and a pruning mask found by linearly combining $S_0 + \l S_1$. In order to help out the random search field for $\l \in (0,\infty)$ we used the knowledge provided by G-COPS and further restricted $\l \in (0,90)$, as $80 < \la < 90$ for the optimal $\la$ found by G-COPS for $\ak = 0.99999$. We tested $14$ different hyperparameters with $5$ random seeds each, the same number as for the manually designed grid-like search. The random parameters were chosen uniformly i.i.d. with $\ak \in (0,1)$ and $\l \in (0,90)$. 
	The results in Table \ref{tab:rnd_search} show that the smaller interval $(0,1)$ for $\ak$, together with G-COPS' robustness shown in Section \ref{subsec:stability}, clearly speeds up the finding of a well performing $\ak$, compared to $\l$ for the linear combination. Naively combining scores linearly does not find an adequate $\l$ in $14$ tries. Actually, the best performing hyperparameter found by the linear combination is worse than the worst one tested for G-COPS, showing that G-COPS uses hyperparameters $\ak$ in a better range than $\l \geq 0$.%In Section \ref{subsec:stability}, we show that pruning with Algorithm \ref{alg:cops} is robust against small variations in $\ak$. Therefore, less hyperparameters $\ak$ are likely needed to achieve good results compared to $\l \geq 0$. 
	%-------------------------------------------------------------------------
	\section{Conclusions}
	
	Training sparse architectures from scratch requires to trade off between
	different criteria for evaluating the pruned network before training.
	Both SOTA one-shot pruning
	methods applied before training starts, SNIP and GraSP, are
	focused on fulfilling only one of these criteria. 
	On various tasks,
	network architectures, pruning rates and combinations of pruning scores we have shown that combining
	two pruning methods via COPS leads to better results than using a single one. Especially
	for high pruning rates, 
	modifying GraSP only slightly with SNIP can surpass GraSP's result considerably.
	However, constraining the gradient flow only improves SNIP marginally for smaller pruning rates. %as COPS introduces an additional hyperparameter compared to SNIP and GraSP. %Compared to SNIP and GraSP, COPS introduces an additional hyperparameter $\ak$. %Thus, using SNIP without control is more economic for smaller pruning rates.  
	The introduced $\ak$ is an interpretable hyperparameter, controlling the strictness of COPS' constraint. Therefore, COPS is more efficient in finding a well performing hyperparameter than naively combining two scores via a linear combination.
	COPS defines a pruning mask as a solution of a LP. Algorithm \ref{alg:cops} provides a method to compute this mask 
	faster than the best reported expected time for solving a
	LP for an arbitrary, dense constraint matrix \cite{jiang_2020}. This might be of intensified interest if more than one pruning step is performed as
	COPS can also be applied to iterative pruning. But also for structured
	methods, where pruning is quite restrictive, combining pruning scores raises new opportunities. Likewise, combining more than two pruning scores seems a possibility to improve
	sparse training even further.
	%\IEEEtriggeratref{14}
	%{\small
	\bibliographystyle{IEEEtran}
	\bibliography{ref_short}
	%}
\end{document}